\pdfoutput=1
\documentclass{amsart}
\usepackage{amsmath,amsthm,amssymb,bbm,mathrsfs,color,graphicx,url,amsaddr}

\newtheorem{proposition}{Proposition}
\newtheorem{theorem}{Theorem}

\newcommand{\rd}{{\mathrm d}}
\newcommand{\KL}{{\mathrm {KL}}}
\newcommand{\Hess}{{\mathrm {Hess}}}
\newcommand{\Id}{{\mathrm {Id}}}
\newcommand{\J}{{\mathrm {J}}}

\title{Forward-Euler time-discretization for Wasserstein gradient flows can be wrong}

\author{Yewei Xu, Qin Li}
\address{Department of Mathematics, University of Wisconsin-Madison, Madison, WI 53706}
\email{xu464@wisc.edu, qinli@math.wisc.edu}

\date{June 2024}

\begin{document}

\begin{abstract}
In this note, we examine the forward-Euler discretization for simulating Wasserstein gradient flows. We provide two counter-examples showcasing the failure of this discretization even for a simple case where the energy functional is defined as the KL divergence against some nicely structured probability densities. A simple explanation of this failure is also discussed.
\end{abstract}

\maketitle

\smallskip
\noindent \textbf{Keywords.} optimal transport, numerical PDEs, numerical methods

\section{Introduction}
Minimizing a functional over the space of probability measures is a topic that has garnered tremendous amount of interests over the past few years. The optimization problems emerging from practice are increasingly challenging, and many take the form over the probability measure space:
\begin{equation}\label{eqn:min_E}
    \min_{\rho\in\mathcal{P}(\Omega)} F[\rho]
\end{equation}
where $\mathcal{P}(\Omega)$ is the collection of all probability measure over the domain $\Omega\subset\mathbb{R}^d$, $F$ is a functional that maps a probability measure to a scalar. Our task is to find, among infinite many possibilities, the probability measure that gives the least value for $F$. This task can be viewed as an extension from the classical optimization problem usually posed over the Euclidean space $\mathbb{R}^d$. As a counterpart of the classical gradient descent over $\mathbb{R}^d$, we now formulate the gradient flow:
\begin{equation}\label{eqn:GF}
\partial_t\rho = -\nabla_{\mathfrak{m}}F[\rho] = \nabla\cdot\left(\rho\nabla\frac{\delta F}{\delta\rho}\right)\,,\tag{WGF}
\end{equation}
where $\mathfrak{m}$ stands for the chosen metric over $\mathcal{P}$. If we confine ourselves to $\mathcal{P}_2$, the class of probability measures whose second moments are finite, the Wasserstein $W_2$ metric can be deployed, and the gradient can be expressed explicitly.

The structure of the PDE strongly suggests a particle treatment on the numerical level. Indeed the term $-\nabla\frac{\delta F}{\delta\rho}$ can be simply interpreted as the velocity field. Let $X_0$ be drawn from $\rho_0$, the initial distribution, then driven by this velocity field:
\begin{equation}\label{eqn:GF_particle}
\dot{X}_t=-\nabla\left.\frac{\delta F}{\delta\rho}\right|_{\rho_t}(X_t)\quad\Rightarrow\quad \mathrm{Law}(X_t)=\rho_t\,\quad\text{for all time}\,.
\end{equation}
In the equation $\nabla\left.\frac{\delta F}{\delta\rho}\right|_{\rho_t}$ means the gradient is evaluated at $\rho_t$, the solution to~\eqref{eqn:GF} at time $t$. Often in practice, this equation cannot be executed since $\rho_t$ is not known ahead of time. So numerically the dynamics is replaced by an ensemble of particles (sometimes also termed the particle method):
\[
\rho_t\approx\frac{1}{N}\sum_i\delta_{x_i(t)}\,,\quad\text{with}\quad\dot{x}_i=-\nabla\left.\frac{\delta F}{\delta\rho}\right|_{\rho_t}(x_i(t))\,.
\]
This is now a coupled ODE system of size $Nd$. Running ODEs has been a standard practice. Most literature choose an off-the-shelf solver, and the simplest among them is the classical Forward Euler (shorthanded by FE throughout the paper). Setting $h$ to be a very small time-stepping, we can march the system forward in time:
\begin{equation}\label{eqn:FE_particle}
    {x}_i^{n+1}=x_i^n-h\nabla\left.\frac{\delta F}{\delta\rho}\right|_{\rho^n}(x^n_i)\,,
\end{equation}
where the superindex stands for the time step, and the lowerindex stands for the choice of the particle. This numerical strategy is vastly popular: Potentially for its simplicity and the lack of counter-argument, it has been the go-to method in many engineering problems, see~\cite{AF21, WCL22, YZL23, YENM23, JLWYZ24, NNSSR24, HTT24}.

We would like to alarm, in our very short note, that this is probably a dangerous practice. In particular, we will give two counter-examples, both framed in the very simple setting where $F$ is a KL divergence against a nicely constructed target distribution, and show that~\eqref{eqn:FE_particle} could lead to erroneous answer. In fact, we will go back to a even more fundamental formula~\eqref{eqn:GF_particle} and show the forward Euler conducted at this level is problematic.

Both examples use the KL divergence, with the target distribution being very smooth, and log-concave. The specific introduction of the problem has different origins, but they share similar features. FE brings the cascading effect along iteration: It decreases the smoothness of the potential in each iteration, and thus after finite rounds of iteration, the distribution at hand is no longer smooth and no longer has finite Wasserstein derivative. This loss of regularity can be introduced by the non-injectivity of the pushforward map, or can be rooted in the consumption of two derivatives in the update, as will be demonstrated in the two examples respectively.

This discovery might seem surprising at the first glance, but it really resonates the classical discovery in numerical PDEs, where most of solvers perform discretization in space before that in time~\cite{L07}. This includes the famous ``Method of lines"~\cite{HSG09} and ``spectral method"~\cite{HGG07}. At the heart of the thought is usually termed the stability theory: time discretization is restricted by the largest spectrum of the PDE operator at hand. Since PDE operators are typically unbounded, discretization has to be done in space first, to threshold the largest spectrum, allowing finite time-stepping. This thinking is somewhat hidden in particle-method for gradient-flow, but the restriction should not be overlooked.

This discovery also brings us back to the classical proposal in the seminar work of~\cite{JKO98}, where the authors essentially proposed an implicit Euler solver for advancing~\eqref{eqn:GF}. Largely recognized as non-solvable, efforts have been made towards developing alternatives~\cite{P15, SKL20, LLW20, MKLGSB21, JLL21, CMW21, FOL23, LWL23, YHY24}. Many such alternatives need to go through similar test on stability, but this task is beyond the scope of the current note.

Section~\ref{sec:prelim} collects some classical derivation on Wasserstein derivative. Readers familiar with the matter should feel safe to skip it. Section~\ref{sec:counterexamples} provides the two counter-examples and the shared general theory that we summarize.

\section{Basics and notations for gradient flow}\label{sec:prelim}
We start off with some quick summary of preliminary results.

The problem under study~\eqref{eqn:min_E} is an optimization over the probability measure space, and it is viewed as an extension of optimization over the Euclidean space, $\min_{x\in\mathbb{R}^d} E(x)$. As a consequence, it is expected that many techniques for $\min_{x\in\mathbb{R}^d} E(x)$ can be carried over. The simplest gradient descent finds $x$ by evolving the following ODE:
\[
\dot{x}=\frac{\rd}{\rd t}x = -\nabla_xE\,,\quad\text{so}\quad
\lim_{t\to\infty}E(x(t))=\min_xE(x)\,,
\]
when some properties of $E$ are taken into account, non-asymptotic convergence rate can also be derived.

To translate this strategy to solve problem~\eqref{eqn:min_E} is not entirely straightforward. The challenge here is two-folded:
\begin{itemize}
    \item[--] $\mathcal{P}(\Omega)$ is infinite dimensional. In the classical optimization problem, $x\in\mathbb{R}^d$ and is finite-dimensional. $\rho\in\mathcal{P}(\Omega)$ is an infinite dimensional object. To push $\rho$ around necessitates the language of partial differential equation. The PDE will characterize the evolution of $\rho(t,\cdot)$ over $\mathcal{P}(\Omega)$ in time.
    \item[--] $\mathcal{P}(\Omega)$ is a nonlinear space. Unlike the Euclidean space and the Hilbert function space that can be ``normed," $\mathcal{P}(\Omega)$ is nonlinear, and thus is a manifold. Without a proper definition of the metric, even the term ``gradient" is not valid. Due to the nonlinearity, metrics have to be defined locally to measure the distance between two probability measures. Many choices are available~\cite{A16,PW24}. Among them it has been prevalent to deploy the Wasserstein metric~\cite{FZMAP15, SQZY18, BJGR19, GK23}. It is a term that is derived from optimal transport theory that measures the length of geodesics between two probability measures. On this metric, one can define the local tangent plane and gradient with respect to this metric. This is the metric that we will use throughout the paper.
\end{itemize}
Under this metric, it is a standard practice to lift up gradient descent to gradient flow, and we arrive at~\eqref{eqn:GF}.

\subsection{Basic Notations}
Throughout the paper, we study objects within $\mathcal{P}_2(\mathbb{R}^d)$, the collection of probability measures of finite second moment on $\mathbb{R}^d$. This space has an important subset $\mathcal{P}_2^r(\mathbb{R}^d)$ that collects all probability measures that are Lebesgue absolutely continuous. Then for any $\rho \in \mathcal{P}_2^r(\mathbb{R}^d)$, we can associate it with a probability density function (pdf) $p : \mathbb{R}^d \to [0,+\infty)$ such that $\rd\rho = p\rd{x}$. For easier notation, we write $p \propto \exp(-U)$ for some $U : \mathbb{R}^d \to \mathbb{R}$. We may use $\rho$ and $\exp(-U)$ interchangeably when the context is clear. The measure $\rho$ is said to be log-concave if $U$ is convex, and log-smooth if $U$ is smooth.

Over the space of $\mathcal{P}_2$, we define a functional $F : \mathcal{P}_2(\mathbb{R}^d) \to \mathbb{R}$. For every $\rho$ define the $W_2$ subgradient of $F$ as:
\[
\nabla \left.\frac{\delta F}{\delta \rho} \right|_{\rho} \,,
\]
i.e., the gradient of the first variation of $F$ with respect to the measure $\rho$.
We denote $D(\vert \partial F \vert)$ the collection of $\rho$ whose subgradient has finite slope in $L_2(\rho)$(\cite[Lemma 10.1.5]{AGS08}), and for $\rho_t\in D(\vert \partial F \vert)$, gradient descent continuous-in-time over the Wasserstein metric gives~\eqref{eqn:GF}.

When FE discretization is deployed, we have the updates:
\begin{equation}\label{eqn:forwardEuler_rho}
\rho_{k+1} = (\text{Id} - h\left.\nabla\frac{\delta F}{\delta \rho}\right\vert_{\rho_k})_\sharp \rho_k \quad\text{ for }\quad k = 0,1,2,\dots \,.\tag{FE}
\end{equation}
with $\rho_k$ standing for the $k$-th iteration solution, $h$ is the time-stepping step-size. $\text{Id}$ is the identity map, and $\sharp$ is the classical pushforward operator. This is the most standard numerical solver used for~\eqref{eqn:GF}, and has been deployed vastly in literature.

According to the definition of the pushforward operator, for any Borel measure $\sigma$ on $X$, $T_\sharp\sigma$ denotes a new probability measure that satisfies:
\[
(T_\sharp \sigma) (A) = \sigma(T^{-1}(A)) \quad\text{ for any}\quad A \subseteq Y\,,
\]
when the given $T : X \to Y$. Here $T$ does not have to be invertible, and $T^{-1}$ only refers to the pre-image. As a consequence, if both measures have densities that are denoted by $\rd \left(T_\sharp\sigma\right) = q\rd y$ and $\rd \sigma = p\rd x$, then:
\[
(T_\sharp \sigma) (A) = \int_A q(y)\rd{y} = \int_{T^{-1}(A)} p(x) \rd x = \sigma(T^{-1}(A))\,,
\]
When $T$ is a diffeomorphism, by the change of variable formula for measures,
\begin{equation}\label{eqn:density_push}
q(y) = p(T^{-1}(y)) \vert \J_{T^{-1}}(y) \vert\,.
\end{equation}
where $\J_{T^{-1}}$ stands for the Jacobian of the inverse map of $T$, and $\vert \cdot \vert$ means taking the absolute value of the determinant of the given matrix.

\subsection{KL Divergence as the energy functional}
We pay a special attention to the energy functional that has the form of the Kullback-Leibler(KL) divergence.

Given two probability measures $\mu,\gamma \in \mathcal{P}_2(\mathbb{R}^d)$, the Kullback-Leibler(KL) divergence of $\mu$ with respect to $\gamma$ is defined as:
\begin{equation}
\text{KL}(\mu \vert \gamma) := \begin{cases}
\int_{\mathbb{R}^d} \frac{d \mu}{d \gamma} \ln(\frac{d \mu}{d \gamma}) d \gamma & \text{if } \mu \ll \gamma\,, \\
+\infty & \text{Otherwise}\,,
\end{cases}
\end{equation}
where $\frac{d \mu}{d \gamma}$ is understood in the sense of Radon-Nikodym derivatives. If in addition, $\mu$ and $\gamma$ are absolutely continuous with density functions denoted as $e^{-V}$ and $e^{-U}$ respectively for some differentiable $U$ and $V$, then
\[
F(\mu):= \text{KL}(\mu \vert \gamma)= \int_{\mathbb{R}^d} (U-V)\exp(-V) \rd x\,,
\]
and the subgradient (velocity field) is explicitly solvable:
\begin{equation}\label{eqn:subgradient}
\nabla \left.\frac{\delta F}{\delta \mu}\right|_\mu(x) = -\nabla V(x) + \nabla U(x) \ .
\end{equation}

\section{Counter-examples}\label{sec:counterexamples}
We provide two counter-examples in subsection~\ref{sec:examples} to invalidate the approach of~\eqref{eqn:forwardEuler_rho}. The failure of~\eqref{eqn:forwardEuler_rho} is rooted in the loss of regularity. According to~\cite{AGS08} Section 10, only when $\rho\in D(|\partial F|)$ can we properly define the differential and thus run~\eqref{eqn:GF} forward in time. We will show that $\rho_k$ will drop out of the set even if $\rho_0$ is prepared within, putting~\eqref{eqn:GF} to halt. There are two sources of this error, and the two examples showcase them respectively. One reason for the loss of regularity is the non-injectivity of the pushforward map used to march to $\rho_{k+1}$ from $\rho_k$. The map sends multiple $x$'s to the same $y$, the density accumulates in a certain region, and a jump discontinuity is introduced at the boundary. The second reason for the loss of regularity is the consumption of two derivatives in the pushforward map. As a consequence, every $\rho_k$ has two fewer derivatives than its prior. If $\rho_0$ has finite regularity, then in finite time, $\rho_k$ drops out of $D(\vert\partial F\vert)$, halting the running of~\eqref{eqn:GF}. This argument can be made more precise and we discuss a proposition in subsection~\ref{sec:regularity}.

These results suggest that even though FE discretization is intuitive, the failure is universal.

\subsection{Examples}\label{sec:examples}

The two examples shown below respectively point towards two sources of loss of regularity.
\subsubsection*{Example 1: Loss of regularity due to the noninjectivity of the pushforward map}
In the first example, we set
\[
F(\rho) = \KL(\rho \vert e^{-U})\,,\quad\text{with}\quad U(x) = \frac{x^2}{2} + \frac{x^4}{4} + C_0 \,.
\]
This functional is well defined over $\mathcal{P}_2^r(\mathbb{R})$, and the constant $C_0$ is chosen to normalize the measure:
\[
C_0 = \ln\left(\int_{\mathbb{R}} \exp(-\frac{x^2}{2} - \frac{x^4}{4})\rd x\right) \,.
\]
We will show the failure of the~\eqref{eqn:forwardEuler_rho} approach. Setting the initial condition $\rho_0$ to be the very simple Gaussian function, by marching forward by just one step of size $h$, we demonstrate that $\rho_1 \not \in D(\vert \partial F \vert)$.

\begin{proposition}
Assume the initial distribution $\rho_0(x)$ has the density of
\[
p_0(x)=\frac{1}{\sqrt{2 \pi}} \exp(-\frac{x^2}{2})\,,
\]
then for any $h>0$, the one-time-step solution $\rho_1$ also has a density $p_1$, with $p_1$ being discontinuous, and thus $\rho_1 \not \in D(\vert \partial F \vert)$, stopping the application of~\eqref{eqn:forwardEuler_rho} to produce $\rho_2$.
\end{proposition}
\begin{proof}
According to~\eqref{eqn:forwardEuler_rho},
\[
\rho_1 = T_\sharp \rho_0 \quad\text{ with }\quad  T(x) := x - h \left.\nabla \frac{\delta F}{\delta \rho}\right|_{\rho_0}(x) \ .
\]
We can explicitly compute $p_1$. In particular, calling~\eqref{eqn:subgradient}, and noticing
\[
V(x) = \frac{x^2}{2} + \ln(2\sqrt{\pi})  \ ; \ U(x) = \frac{x^2}{2} + \frac{x^4}{4} + C_0\,,
\]
we have:
\begin{equation}\label{eqn:T_ex1}
T(x)= x - h (\nabla U(x) - \nabla V(x)) = x - hx^3\,.
\end{equation}
It is important to notice that $T$ is not one-to-one, as seen in Figure~\ref{fig:pushforward_T}.
\begin{figure}[htb]
    \centering
    \includegraphics[width = 0.48\textwidth]{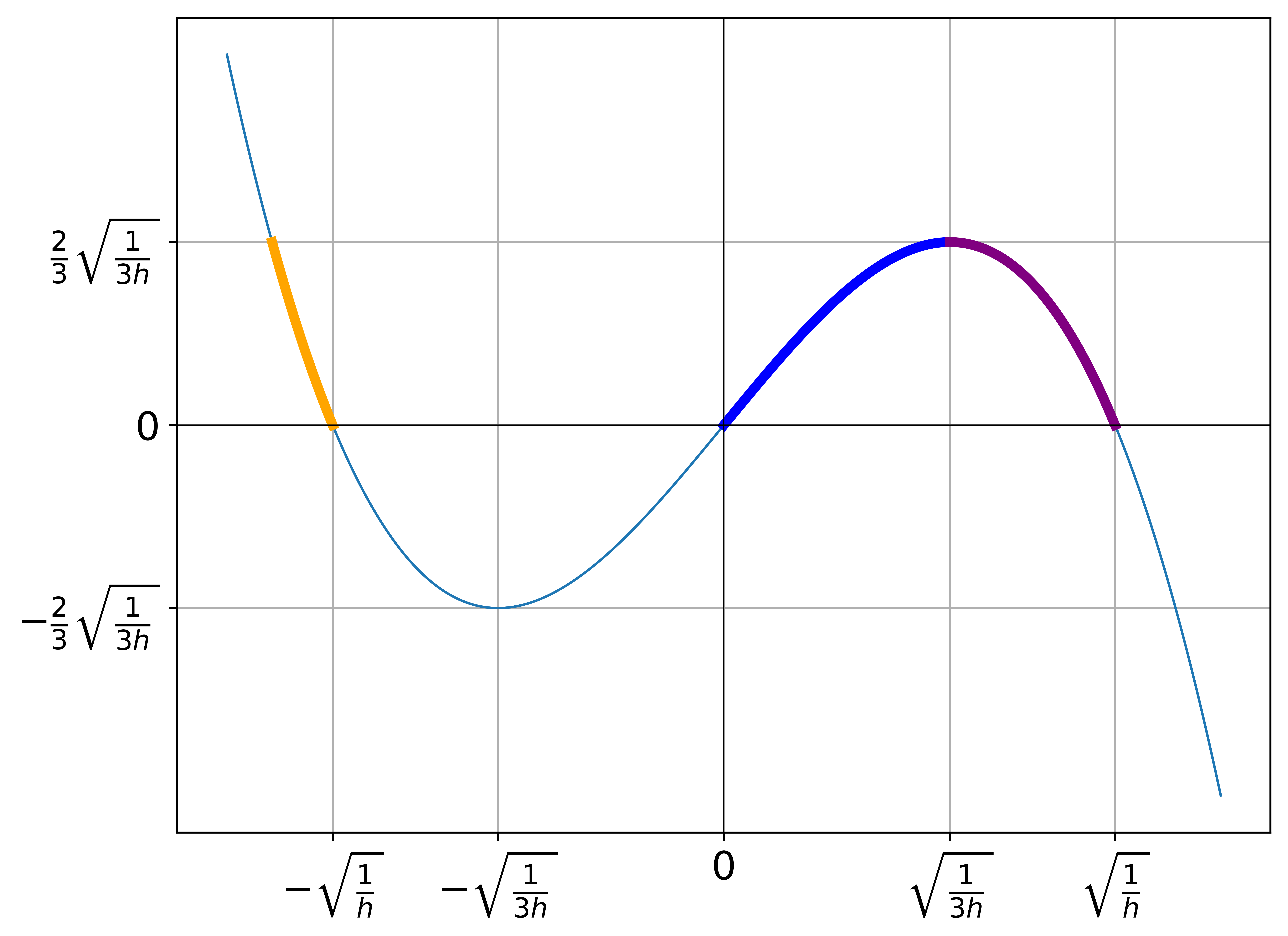}
    \includegraphics[width = 0.48\textwidth]{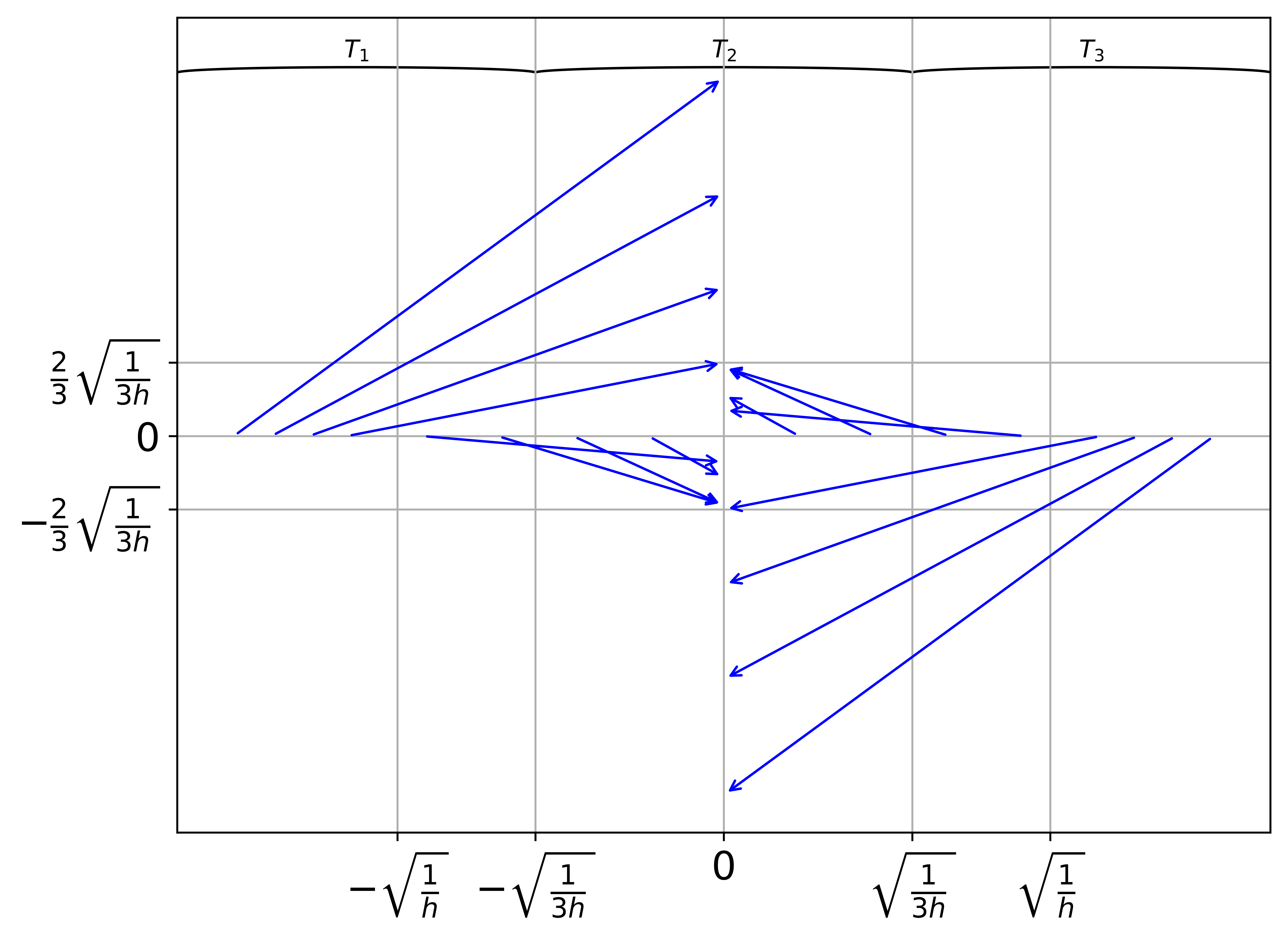}
    \caption{The plot on the left shows the pushforward function $T$. Note that every $y \in (0,\frac{2}{3}\sqrt{\frac{1}{3h}})$ has three pre-image points of $x$. For example, for $y\in[0,\frac{2}{3}\sqrt{\frac{1}{3h}}]$, the different parts of pre-image are bold-lined: $(-\infty,-\sqrt{\frac{1}{3h}}]$, $[-\sqrt{\frac{1}{3h}},\sqrt{\frac{1}{3h}}]$ and $[\sqrt{\frac{1}{3h}},\infty)$ respectively. The plot on the right shows the mapping of $T$. When $T$ is confined in three parts of the region, denoted as $T_1, T_2, T_3$, injectivity is resumed.}
\label{fig:pushforward_T}
\end{figure}

In particular, every point in the range of $y \in (0,\frac{2}{3}\sqrt{\frac{1}{3h}})$ has three pre-image points of $x$. Only confined in smaller $x$-regions, injectivity is resumed and inversion of $T$ is well-defined, see Table~\ref{tab:my_label}.
Recall that $T(x) = x - hx^3$, so for any given pair $(x,y)$ such that $T_i(x) = T(x) = y$, we have the determinant of the Jacobian being $\vert \text{J}_{T_i}(x)\vert = \vert \text{J}_{T}(x)\vert = \vert 1 - 3 h x ^2 \vert$. When $x \neq \pm\sqrt{\frac{1}{3h}}$, we define:
\[
\vert \text{J}_{i}(y)\vert \doteq \vert \text{J}_{T^{-1}_i(y)}\vert=\frac{1}{\vert \text{J}_{T_i}(x)\vert} = \frac{1}{\vert 1 - 3 h x ^2 \vert} \ .
\]

\begin{table}[h!]
    \centering
    \begin{tabular}{|c|c|c|}
    \hline
    Map & domain of $x$ & range of $y$ \\
    \hline
         $T_1$ & $(-\infty, -\sqrt{\frac{1}{3h}})$ & $y \in (-\frac{2}{3}\sqrt{\frac{1}{3h}},+\infty)$\\
         $T_2$ & $(-\sqrt{\frac{1}{3h}},\sqrt{\frac{1}{3h}})$ & $y \in (-\frac{2}{3}\sqrt{\frac{1}{3h}},\frac{2}{3}\sqrt{\frac{1}{3h}})$ \\
         $T_3$ & $(\sqrt{\frac{1}{3h}},+\infty)$ & $y \in (-\infty, \frac{2}{3}\sqrt{\frac{1}{3h}})$ \\
         \hline
    \end{tabular}
    
    \caption{Range of $T$ in different parts of domains of injectivity.}\label{tab:my_label}
\end{table}

The loss of such injectivity prevents us to directly apply~\eqref{eqn:density_push}, however, similar derivative nevertheless can be carried out. We focus on $y>0$ and summarize the calculations here:
\begin{itemize}
\item At $y=\frac{2}{3}\sqrt{\frac{1}{3h}}$, there are two points in the pre-image. Since $\rho_0$ is Lebesgue absolutely continuous, $\rho_1(\{\frac{2}{3}\sqrt{\frac{1}{3h}}\}) = 0$.

\item In $(\frac{2}{3}\sqrt{\frac{1}{3h}},\infty)$, every $y$ has one pre-image point given by $T_{1}^{-1}$.
Deploying~\eqref{eqn:density_push}, we have:
\begin{equation}\label{eqn:p1_ex1_a}
p_1(y) = p_0(T_1^{-1}(y)) \vert \J_1(y) \vert \,.
\end{equation}

\item In $(0,\frac{2}{3}\sqrt{\frac{1}{3h}})$, every $y$ has three pre-image points given by $T_{1}^{-1}$, $T_{2}^{-1}$ and $T_{3}^{-1}$ respectively \footnote{The pushforward measure on $(0,\frac{2}{3}\sqrt{\frac{1}{3h}})$ is equal to $(T_1)_\sharp (\rho_0\left.\right\vert_{T_1^{-1}(\frac{2}{3}\sqrt{\frac{1}{2h}},-\sqrt{\frac{1}{h}})}) + (T_2)_\sharp (\rho_0\left.\right\vert_{(0,\sqrt{\frac{1}{3h}})}) + (T_3)_\sharp (\rho_0\left.\right\vert_{(\sqrt{\frac{1}{3h}},\sqrt{\frac{1}{h}})})$, where $\rho_0\left.\right\vert_{A})$ denotes the measure confined on the set $A$. Each of $T_1, T_2, T_3$ confined on their domains is a diffeomorphism, and therefore we may apply the change of variable formula to compute the densities of the pushforward measures.}. The new density is:
\begin{equation}\label{eqn:p1_ex1_b}
\begin{aligned}
& p_1(y) =& p_0(T_1^{-1}(y)) \vert  \J_1(y) \vert + p_0(T_2^{-1}(y)) \vert  \J_2(y) \vert + p_0(T_3^{-1}(y)) \vert \J_3(y) \vert  \,.
\end{aligned}
\end{equation}
\end{itemize}
Noting that $\lim_{y \uparrow \frac{2}{3}\sqrt{\frac{1}{3h}}} T_2^{-1}(y) = \sqrt{\frac{1}{3h}} = \lim_{y \uparrow \frac{2}{3}\sqrt{\frac{1}{3h}}} T_3^{-1}(y)$, so
\begin{equation}\label{eqn:limit_p_singular}
\lim_{y \uparrow \frac{2}{3}\sqrt{\frac{1}{3h}}} p_0(T_2^{-1}(y)) = p_0(\sqrt{\frac{1}{3h}}) = \lim_{y \uparrow \frac{2}{3}\sqrt{\frac{1}{3h}}} p_0(T_3^{-1}(y)) >0\,,
\end{equation}
and
\begin{equation}\label{eqn:limit_J_singular}
\lim_{y \uparrow \frac{2}{3}\sqrt{\frac{1}{3h}}} \vert \J_2(y)\vert = \lim_{x \uparrow \sqrt{\frac{1}{3h}}} \frac{1}{\vert 1 - 3 h x^2 \vert} = +\infty = \lim_{x \downarrow \sqrt{\frac{1}{3h}}} \frac{1}{\vert 1 - 3 h x^2 \vert} = \lim_{y \uparrow \frac{2}{3}\sqrt{\frac{1}{3h}}} \vert \J_3(y)\vert\,.
\end{equation}
Plugging~\eqref{eqn:limit_p_singular}-\eqref{eqn:limit_J_singular} back in~\eqref{eqn:p1_ex1_b}, we notice a blow-up solution at $y=\frac{2}{3}\sqrt{\frac{1}{3h}}$\footnote{We should note that it roughly behaves as $\frac{1}{\sqrt[3]{\epsilon}}$ for small $\epsilon$ and is an integrable singularity.}. Comparing~\eqref{eqn:p1_ex1_a} and~\eqref{eqn:p1_ex1_b}, it is now immediate that $p_1(y)$ has a jump discontinuity at $y = \frac{2}{3}\sqrt{\frac{1}{3h}}$, which means the density associated with the probability measure $\rho_1$, $p_1 \notin W_{\text{loc}}^{1,1}(\mathbb{R})$. By \cite[Theorem 10.4.9]{AGS08}, $\rho_1 \not \in D(\vert \partial F \vert)$. The same derivation applies to $p_1(-\infty,-\frac{2}{3}\sqrt{\frac{1}{3h}})$. 
\end{proof}

In this example, the initial condition is a very nice distribution, and the target distribution is log-concave and smooth. Problem is introduced by the non-injectivity of the pushforward map $T$. More precisely, for $y \in (\frac{2}{3}\sqrt{\frac{1}{3h}}, +\infty)$, only the part $T_1$ contributes to the density here.
For $y = (0,\frac{2}{3}\sqrt{\frac{1}{3h}})$, however, all of $T_1, T_2, T_3$ contribute to the measure. Since the pushforward densities of $T_2$ and $T_3$ are not vanishing at the border of $y = \frac{2}{3}\sqrt{\frac{1}{3h}}$, a jump discontinuity is introduced. Such jump discontinuity is observed for any finite $h$. As $h\to0$, according to~\eqref{eqn:T_ex1}, $T=x$, also seen in Figure~\ref{fig:pushforward_h_comp}. The map converges to the identity map, and the distribution does not move. The problem only occurs to discrete-in-time setting introduced by running~\eqref{eqn:forwardEuler_rho}.

\begin{figure}[htb]
    \centering
    \includegraphics[width = 0.48\textwidth]{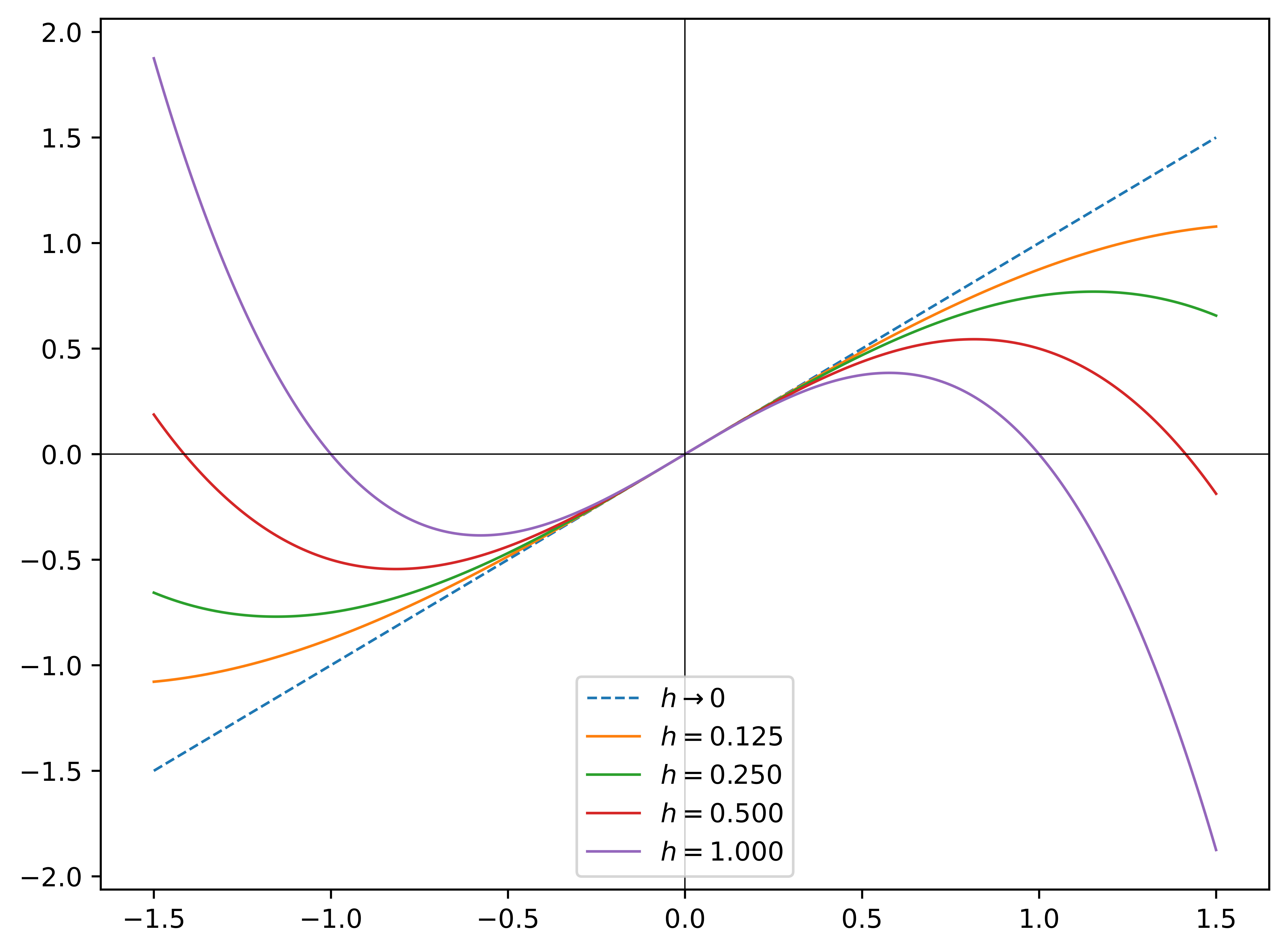}
    \caption{Plot of pushforward functions $T$ for different step sizes $h$. In the $h\to0$ limit, $T(x)=x$.}
\label{fig:pushforward_h_comp}
\end{figure}

\subsubsection*{Example 2: Loss of regularity due to the consumption of derivatives.}

The second example share similar features as the first, but the irregularity is introduced through a different manner. The problem is set as:
\[
F(\rho) = \KL(\rho \vert e^{-U})\,,\quad\text{with}\quad U(x) = \frac{x^2}{2} + \ln(2\sqrt{\pi})\,.
\]
Here the constant $\ln(2\sqrt{\pi})$ takes care of the normalization. As in the previous example, the energy function is well defined over $\mathcal{P}_2^r(\mathbb{R})$. According to~\eqref{eqn:forwardEuler_rho}, the pushforward for every iteration is:
\begin{equation}\label{eqn:FE_rho_ex2}
\rho_{k+1} = (T_k)_\sharp \rho_k \quad\text{ with }\quad T_k(x) := x - h_k \left.\nabla \frac{\delta F}{\delta \rho}\right\vert_{\rho_k}(x) \,,
\end{equation}
where $h_k \in (0,1)$ is the step-size at $k$-th iteration. We initialize our iteration at $\rho_0$ with the density
\begin{equation}\label{eqn:def_p0_ex2}
p_0(x)=\frac{1}{D_0} \exp(-V_0(x))\,,\quad \text{with}\quad V_0(x) := \begin{cases}
\frac{x^2}{2} & x \in (-1,1) \\
\vert x \vert - \frac{1}{2} & \text{Otherwise}
\end{cases} \,,
\end{equation}
where $D_0$ stands for the normalization coefficient \begin{equation}\label{eqn:def_d0_reg}
D_0 = \int_{\mathbb{R}}V_0(x) d x = \sqrt{2 \pi}\text{Erf}(\frac{1}{\sqrt{2}}) + \frac{2}{\sqrt{e}}
\end{equation}.
The following proposition holds true:
\begin{proposition}\label{prop:ex2_property}
Under the conditions stated above:
\begin{itemize}
    \item $F$ is $1$-convex over $\mathcal{P}_2^r(\mathbb{R})$ with its unique minimizer being $\rho_* = e^{-U}$. The minimum value $F_* = F(\rho_*) = 0$.
    \item The density function $p_1$ of $\rho_1$, is discontinuous. In addition, $\rho_1 \not \in D(\vert \partial F \vert)$.
    \item There are constants $a_k, c_k \in [1,\infty)$ and $b_k\in \mathbb{R}$ so that $p_k(x)$ has the following form:
\begin{equation}\label{eqn:def_pk_ex2}
p_k(x) = \begin{cases}
\frac{1}{D_0}\exp(-\frac{x^2}{2}) & x \in [0,1) \\
0 & x \in [1,c_k) \\
\exp(-a_k x + b_k) & x \in [c_k, +\infty)\\
p_k(-x) & x \in (-\infty,0)
\end{cases}\,.
\end{equation}
\end{itemize}
\end{proposition}

This proposition immediately evokes our theorem:
\begin{theorem}
    Under the conditions as in Proposition~\ref{prop:ex2_property}, $F(\rho_k) > 0.019$ for all $k$.
\end{theorem}
\begin{proof}
According to the Pinsker's inequality, we have
\[
\begin{aligned}
F(\rho_k) - F_* & = \KL(\rho_k \vert e^{-U}) \geq 2 \left(\text{TV}(\rho_k, e^{-U})\right)^2 \\
& \geq 2 \left(\int_{-1}^{1} (p_k(x) - p(x)) \rd x\right)^2 \\
& = 2 \left((\frac{1}{\sqrt{2\pi}} - \frac{1}{D_0}) \int_{-1}^{1} \exp(-\frac{x^2}{2}) \rd x\right)^2 \\
& = 4 \pi \left(\text{Erf}(\frac{1}{\sqrt{2}})\right)^2 \left(\frac{1}{\sqrt{2 \pi}} - \frac{1}{\sqrt{2 \pi}\text{Erf}(\frac{1}{\sqrt{2}}) + \frac{2}{\sqrt{e}}}\right)^2 > 0.019 \,.
\end{aligned}
\]
Here the second inequality comes from lower bounding the total variation using the total variation on the subset $(-1,1)$; the third line comes from plugging in Equation~\eqref{eqn:def_p0_ex2} and Equation~\eqref{eqn:def_pk_ex2}; the last line comes from using the definition of the normalization constant $D_0$ in Equation~\eqref{eqn:def_d0_reg}; and the last inequality comes from a numerical computation.
\end{proof}

\begin{proof}[Proof for Proposition~\ref{prop:ex2_property}]
The first bullet point of the proposition is a direct derivation from~\cite[Theorem 9.4.11]{AGS08}. To show the second bullet point, we compute $p_1$ explicitly. Following~\eqref{eqn:subgradient}, the $W_2$-subgradient at $\rho_0$ is
\[
\left.\nabla \frac{\delta F}{\delta \rho}\right\vert_{\rho_0}(x) = -\nabla V_0(x) + \nabla U(x) = \begin{cases}
x+1 & x \in (-\infty,-1] \\
0 & x \in (-1,1) \\
x-1 & x \in [1,\infty)
\end{cases}\,,
\]
meaning the pushforward map is
$$T_0(x) = x - h_0 \left.\nabla \frac{\delta F}{\delta \rho}\right\vert_{\rho_0} = \begin{cases}
(1 - h_0)x - h_0 & x \in (-\infty,-1] \\
x & x \in (-1,1) \\
(1 - h_0)x + h_0 & x \in [1,\infty)
\end{cases} \,,
$$
for step-size $h_0$. The inverse and the Jacobian can also be computed explicitly:
\[
T_0^{-1}(x) = \begin{cases}
\frac{x + h_0}{1 - h_0}\\
x\\
\frac{x - h_0}{1 - h_0}
\end{cases} \,,\quad \text{J}_{T_0^{-1}}(x) = \begin{cases}
\frac{1}{1 - h_0} & x \in (-\infty,-1] \\
1 & x \in (-1,1) \\
\frac{1}{1 - h_0} & x \in [1,\infty)
\end{cases} \ .
\]
Using the change of variable formula for measures, Equation~\eqref{eqn:density_push} gives
\[
\begin{aligned}
p_1(x) &= p_0(T_0^{-1}(x))\vert \J_{T_0^{-1}}(x)\vert \\
&= \frac{1}{D_0}\begin{cases}
\frac{1}{1 - h_0}\exp(-\frac{x+h_0}{1-h_0}+\frac{1}{2}) & x \in (-\infty,-1] \\
\exp(-\frac{x^2}{2}) & x \in (-1,1) \\
\frac{1}{1 - h_0}\exp(-\frac{x-h_0}{1-h_0}+\frac{1}{2}) & x \in [1,\infty)
\end{cases} \,.
\end{aligned}
\]
It is immediate that $p_1$ has jump discontinuity at $\{-1,1\}$ two points as long as $h_0 > 0$, and therefore it is not in $W_{\text{loc}}^{1,1}(\mathbb{R})$. Hence $\rho_1 \not \in D(\vert \partial F \vert)$ according to \cite[Theorem 10.4.9]{AGS08}.

To prove the third bullet point, we first notice that due to the symmetry, we only need to show the validity of this formula for $x\geq 0$. When $k=0$, from the definition in Equation~\eqref{eqn:def_p0_ex2}, $a_0 = 1, b_0 = \frac{1}{2} - \ln(D_0), c_0 = 1$. Suppose that the claim is true for $k$, then using~\eqref{eqn:subgradient} and~\eqref{eqn:FE_rho_ex2}:
\[
T_k(x) = x - h_k\left.\nabla \frac{\delta F}{\delta \rho}\right\vert_{\rho_k}(x) = \begin{cases}
x & x \in [0,1) \\
(1 - h_k)x + a_k h_k & x \in [c_k,\infty)
\end{cases} \,.
\]
Noting the monotonicity of $T_k$ allows us to compute the image of $T$ explicitly. In particular:
\[
\text{Imag}(\left.T_k\right|_{[c_k,\infty)}) = [c_{k+1},\infty)\,,\quad\text{with}\quad c_{k+1} = (1 - h_k)c_k + a_k h_k \geq 1 + (a_k - 1) h_k \geq 1\,.
\]
As a consequence, inverse and Jacobian can be computed:
\[
T_k^{-1}(x) = \begin{cases}
x & x \in [0,1) \\
\frac{x - a_k h_k}{1 - h_k} & x \in [c_{k+1},\infty)
\end{cases} \,,\quad \ \text{J}_{T_k^{-1}}(x) = \begin{cases}
1 & x \in [0,1) \\
\frac{1}{1 - h_k} & x \in [c_{k+1},\infty)
\end{cases} \,,
\]
leading to
\[
p_{k+1}(x)  = \begin{cases}
\frac{1}{D_0}\exp(-\frac{x^2}{2}) & x \in [0,1) \\
0 & x \in [1,c_{k+1}) \\
\exp(-a_{k+1} x + b_{k+1}) & x \in [c_{k+1}, +\infty)\\
p_{k+1}(-x) & x \in (-\infty,0)
\end{cases} \,,
\]
where $a_k$ and $b_k$ solve:
\[
\begin{aligned}
\exp(-a_{k+1}x+b_{k+1}) & = p_{k}(T_{k}^{-1}(x)) \vert \J_{T_{k}^{-1}}(x) \vert \\
& = \frac{1}{1 - h_k}\exp(-a_k\frac{x-a_kh_k}{1-h_k} + b_k) \\
& = \exp\left(-\frac{a_k}{1-h_k}x + b_k + \frac{a_k^2 h_k}{1 - h_k} - \ln(1 - h_k)\right) \,,
\end{aligned}
\]
meaning
\[
b_{k+1} = b_k + \frac{a_k^2 h_k}{1 - h_k} - \ln(1 - h_k)\,,\quad a_{k+1} = \frac{a_k}{1 - h_k} > a_k \geq 1\,.
\]
This finishes the proof of induction.
\end{proof}

\subsection{Loss of regularity through~\eqref{eqn:forwardEuler_rho} is generic}\label{sec:regularity}

Both examples showcase the reduced regularity as one propagates $\rho_k$ forward in time according to~\eqref{eqn:forwardEuler_rho}. This is a generic property that can be shared across FE solver for gradient flow for all functionals of KL divergence type. We extract this property and formulate the following:

\begin{proposition}
Let $F(\rho) = \KL(\rho \vert e^{-U})$ with $U \in \mathcal{C}^\infty(\mathbb{R}^d)$. Assume $U$ is gradient Lipschitz, meaning $\Hess[U] \preceq M I$ for some $M \in (0,+\infty)$. Let $\rho_0 = e^{-V_0}$ for a smooth potential that is $V_0 \in \mathcal{C}^{m+2}(\mathbb{R})$ and $-M_0 I \preceq \Hess[V_0]$ for some $M_0 \in (0,+\infty)$, then the one-step pushforward
\[
\rho_1 = (\Id - h_0 \left.\nabla \frac{\delta F}{\delta \rho}\right\vert_{\rho_0})_\sharp \rho_0=e^{-V_1}
\]
has to have reduced regularity. Namely, for $h_0 \in (0,\frac{1}{M + M_0})$, $V_1 \in \mathcal{C}^{m}$ has two fewer derivatives than $V_0$.
\end{proposition}

\begin{proof}
According to the formula~\eqref{eqn:forwardEuler_rho} and~\eqref{eqn:FE_rho_ex2}:
\[
T_0(x) = x - h_0 \left.\nabla \frac{\delta F}{\delta \rho}\right\vert_{\rho_0}(x) = x + h_0 \nabla V_0(x) - h_0 \nabla U(x) \,.
\]
Since $V_0 \in \mathcal{C}^{m+2}$, $T_0 \in \mathcal{C}^{m+1}$ and its Jacobian is
\[
\J_{T_0}(x) = I + h_0 \Hess[V_0](x) - h_0 \Hess[U](x) \succ 0 \,.
\]
The positivity of the Jacobian comes from the choice of $h_0$. It also implies that $T_0$ is injective, and thus by the Global Inverse Function Theorem, $T_0^{-1}$ is well-defined and $T_0^{-1} \in \mathcal{C}^{m+1}$. Using the change of variable between pushforward measures Equation~\eqref{eqn:density_push}
\[
\begin{aligned}
\exp\left(-V_1\left(x\right)\right) & = \left(\exp\left(-V_0\left(T_0^{-1}\left(x\right)\right)\right)\right) \left\vert \J_{T_0^{-1}}\left(x\right) \right\vert \\
& = \exp\left(-\left(V_0\left(T_0^{-1}\left(x\right)\right) - \ln\left(\left\vert \J_{T_0^{-1}}\left(x\right) \right\vert\right)\right)\right) \,,
\end{aligned}
\]
namely:
\[
V_1\left(x\right) = V_0\left(T_0^{-1}\left(x\right)\right) - \ln\left(\left\vert \J_{T_0^{-1}}\left(x\right) \right\vert\right) \,.
\]
Recall that $T_0^{-1} \in \mathcal{C}^{m+1}$ and $V_0 \in \mathcal{C}^{m+2}$, and the Jacobian is taking one higher order of derivative on $V_0$, so $\ln\left(\left\vert \J_{T_0^{-1}}\left(x\right) \right\vert\right) \in \mathcal{C}^{m+1}$, and we conclude that $V_1 \in \mathcal{C}^{m}$.
\end{proof}

This proposition is generic: as long as~\eqref{eqn:forwardEuler_rho} is applied through the form of the gradient flow~\eqref{eqn:GF}, $2$ derivatives are lost in every iteration. Any initial data that has finite amount of derivatives will quickly lose hold of their regularity, and falls in the regime where Wasserstein gradient is not even defined. As a consequence, the simple FE stepping should be utilized with caution for gradient flow.

\bibliographystyle{abbrv}   
\bibliography{main}

\end{document}